\documentclass[10pt]{article}
\usepackage{spconf,amsmath,graphicx}
\usepackage{subfigure}
\usepackage{multirow}
\usepackage{amssymb}
\usepackage{placeins}
\usepackage{color}
\usepackage{tikz}
\allowdisplaybreaks[1]

\usepackage[noadjust]{cite}

\usepackage[latin1]{inputenc}
\usepackage{amsthm, amsfonts, amssymb, mathrsfs, amsmath}
\usepackage{nicefrac}
\usepackage{nohyperref}
\usepackage{mathtools}
\usepackage{multirow}
\usepackage[ruled]{algorithm2e}
\usepackage{algorithmic}
\usepackage[draft]{fixme}

\usepackage{tikz}
\allowdisplaybreaks[1]

\newcommand{\bx}{{\boldsymbol{x}}}
\newcommand{\bu}{{\boldsymbol{u}}}
\newcommand{\bv}{{\boldsymbol{v}}}
\newcommand{\bz}{{\boldsymbol{z}}}
\newcommand{\by}{\boldsymbol{y}}

\newcommand{\bhx}{{\boldsymbol{\widehat{x}}}}

\makeatletter
\renewcommand{\@algocf@capt@plain}{above}% formerly {bottom}
\makeatother

\DeclareMathOperator*{\argmin}{arg\,min}

\usepackage{amsthm}
\newtheorem{theorem}{Theorem}[section]

\newtheorem{lemma}[theorem]{Lemma}
\title{\vspace{-15pt}Sparse Signal Reconstruction with Multiple Side Information using Adaptive Weights for Multiview Sources\vspace{-5pt}}

\oneauthor
  {\vspace{-1pt}Huynh Van Luong\textsuperscript{1},  J\"{u}rgen Seiler\textsuperscript{1}, Andr\'{e} Kaup\textsuperscript{1}, S{\o}ren Forchhammer\textsuperscript{2}}{\vspace{-25pt}}
	{\vspace{-1pt}\textsuperscript{1}Multimedia Communications and Signal Processing, Friedrich-Alexander-Universit\"{a}t Erlangen-N\"{u}rnberg,\\ 91058 Erlangen, Germany\\[-1pt]
   \textsuperscript{2}DTU Fotonik, Technical University of Denmark, 2800 Kgs. Lyngby, Denmark}
\begin{document}
%\ninept
%
%\vspace{-30pt}
\maketitle

\begin{abstract}
This work considers reconstructing a target signal in a context of distributed sparse sources. We propose an efficient reconstruction algorithm with the aid of other given sources as multiple side information (SI). The proposed algorithm takes advantage of compressive sensing (CS) with SI and adaptive weights by solving a proposed weighted $n$-$\ell_{1}$ minimization. The proposed algorithm computes the adaptive weights in two levels, first each individual intra-SI and then inter-SI weights are iteratively updated at every reconstructed iteration. This two-level optimization leads the proposed reconstruction algorithm with multiple SI using adaptive weights (RAMSIA) to robustly exploit the multiple SIs with different qualities. We experimentally perform our algorithm on generated sparse signals and also correlated feature histograms as multiview sparse sources from a multiview image database. The results show that RAMSIA significantly outperforms both classical CS and CS with single SI, and RAMSIA with higher number of SIs gained more than the one with smaller number of SIs.
\end{abstract}
%\vspace{-2pt}
\begin{keywords}
Side information, compressive sensing, sparse signal, $n$-$\ell_{1}$ minimization, adaptive weights.
\end{keywords}

%\ifCLASSOPTIONpeerreview
%\begin{center} \bfseries EDICS Category: OPT-SOPT \end{center}
%\fi

%\IEEEpeerreviewmaketitle
%\vspace{-10pt}
\section{Introduction}\label{sec:intro}
\vspace{-7pt}
%\IEEEPARstart
%Advanced technologies have recently empowered new smart tiny devices and visual sensors for a wide range of applications including visual surveillance and mobile augmented reality. These advanced devices are able to address the emerging application challenges by cooperatively using multiple visual sensors. They can also enrich visual information by capturing multi-view images and efficiently interact with the environment information within heterogeneous networks in real-time. However, these smart tiny cameras have limited resources for such time-critical applications. In particular, multi-view object tracking and recognition \cite{Taj2011,ChellappaIEEE08} are challenging in a distributed sensing scenario of cooperatively performing tiny cameras. Gathering and processing information \cite{YangIEEE10} can be further enhanced by utilizing correlations among the multi-view images without communications among them.

Recent emerging applications such as visual sensor surveillance and mobile augmented reality are challenging in a distributed sensing scenario of smart tiny devices within heterogeneous networks in real-time. By cooperatively using multiple sensors, we can further enhance gathering and processing information under time-resource constraints for such devices. In particular, multiview object tracking and recognition \cite{Taj2011,ChellappaIEEE08} are challenging how to utilize correlations among the multiview images without communications among them. The distributed sensing challenge may be addressed by representing it as a distributed sparse representation of multiple sources \cite{YangIEEE10}. This raises the needs for representing and reconstructing the sparse sources along with exploiting the correlated information among them. One of the key problems is how to robustly reconstruct a compressed source given supported observations, i.e., other available sources as SI. Attractive CS techniques \cite{DonohoCOM06,CandesTIT06,Beck09,MotaGLOBALSIP14,MotaARXIV14,Scarlett,MotaICASSP15,Weizman15,WarnellTIP15} may enable these distributed sources to deal with such reconstruction problems.

%%\vspace{-15pt}
Recently, some CS reconstructions with prior information \cite{MotaGLOBALSIP14,MotaARXIV14,Scarlett,MotaICASSP15,Weizman15,WarnellTIP15} have been investigated in terms of theoretical as well as application aspects. It was shown that if the prior information is good enough, a $l_{1}$-$l_{1}$ minimization \cite{MotaGLOBALSIP14,MotaARXIV14}, which integrates SI into classical CS, improves the reconstruction dramatically. The work in \cite{Scarlett,MotaICASSP15,Weizman15,WarnellTIP15} shows efficient reconstructions with SI for sparse signals from a limited number of measurements in background subtraction and MRI imaging applications. These schemes are restricted to considering only one SI in conditions of good enough qualities. However, we are interested in the multiview heterogeneous sources with quality changing in time, i.e., unknown prior correlations, correspondences, and SI qualities among sources. Therefore, we are aiming at improving the reconstruction given arbitrary SIs, whose possible poor qualities and the multiple SIs are taken into account.

In this work, we propose an efficient reconstruction algorithm with multiple SIs using adaptive weights by solving a weighted $n$-$\ell_{1}$ problem. The algorithm adaptively reconstructs a sparse source from a small number of measurements, which are obtained by a random projection, given other sources. Our previous work \cite{LuongDCC16} proposed a reconstruction algorithm with multiple SI, where the weights are iteratively updated for each given individual index cross all SIs under only its local constraint. To improve performance, the algorithm we propose here contributes to solving the $n$-$\ell_{1}$ problem that computes the adaptive weights in two levels. Intra-SI weights are first computed in each individual SI then global inter-SI weights are updated at every iteration of the reconstruction. By this way, the proposed algorithm can robustly take advantage of exploiting both intra-source and inter-source correlations to adapt to on-the-fly correlation changes in time, e.g., multiview images with different SI qualities.

The rest of this paper is organized as follows. Section \ref{relatedWork} reviews related works including fundamental problems of signal recovery and CS with SI. Our problem statement and proposed algorithm are presented in Sec. \ref{RAMSIA}. We demonstrate performances of the proposed algorithm on distributed sparse sources in Sec. \ref{Experiment}.%which are extracted from a multiview image database as well as generated signals
\vspace{-10pt}
\section{Related Work}
\vspace{-7pt}
\label{relatedWork}
%\subsection{Compressive }
%\label{DCforMultiview}
In this section, we review a fundamental problem of signal recovery from low-dimensional signals \cite{DonohoCOM06,CandesTIT06,Beck09} and CS with SI \cite{MotaGLOBALSIP14,MotaARXIV14,Scarlett,MotaICASSP15,Weizman15,WarnellTIP15}.
%%\vspace{-5pt}
%\subsection{Signal Recovery}
%%\vspace{-5pt}
%\label{fundamentalRecovery}

\textbf{Signal Recovery}. Low-dimensional signal recovery arises in a wide range of applications in signal processing. Most signals in such applications have sparse representations in some domain. Let $\bx\hspace{-2pt}\in\hspace{-2pt}\mathbb{R}^{n}$ denote a sparse source, which is indeed compressible. The source $\bx$ can be reduced by sampling via a projection \cite{CandesTIT06}. We denote a random measurement matrix for $\bx$ by $\mathbf{\Phi}\hspace{-2pt}\in \hspace{-2pt}\mathbb{R}^{ m\times n} (m\hspace{-2pt}< \hspace{-2pt}n)$, whose elements are sampled from an i.i.d. Gaussian distribution. Thus, we get a compressed vector $\by\hspace{-2pt}=\hspace{-2pt}\mathbf{\Phi}\bx$, also called measurement, consisting of $m$ elements. The source $\bx$ can be recovered \cite{CandesTIT06,DonohoCOM06} by solving:
%\vspace{-7pt}
%%\vspace{-5pt}
\begin{equation}\label{l1-norm}
    \min_{\bx} ||\bx||_{1} \mathrm{~subject~to~} \by\hspace{-2pt}=\hspace{-2pt}\mathbf{\Phi}\bx,
%\vspace{-7pt}
\end{equation}
where $||\bx||_{p}\hspace{-2pt}:=\hspace{-2pt}(\sum_{i=1}^{n}|x_{i}|^{p})^{1/p}$ is $\ell_{p}$ norm of $\bx$ wherein $x_{i}$ is an element of $\bx$.

Problem \eqref{l1-norm} becomes finding a solution to
%\vspace{-10pt}
%\begin{equation}\label{l1-regularization}
%    \min_{\bx}\Big\{H(\bx) = \frac{1}{2}||\mathbf{\Phi}\bx-\by||^{2}_{2} + \lambda ||\bx||_{1}\Big\},
%    \vspace{-7pt}
%\end{equation}
%where $\lambda>0$ is a regularization parameter. For a general model, we consider a general problem:
%\vspace{-7pt}
\begin{equation}\label{l1-general}
    \min_{\bx}\{H(\bx) = f(\bx) + g(\bx)\},
    %\vspace{-7pt}
\end{equation}
where $f\hspace{-2pt}:=\hspace{-2pt}\mathbb{R}^{n} \hspace{-2pt}\rightarrow\hspace{-2pt}\mathbb{R}$ is a smooth convex function with Lipschitz constant $L_{\nabla f}$ \cite{Beck09} of gradient $\nabla f$
% i.e., $||\nabla f(\bu)-\nabla f(\bv)||_{2}\leq L_{\nabla f}||\bu-\bv||_{2}$ for all $\bu,\bv\in\mathbb{R}^{n}$,
and $g\hspace{-2pt}:=\hspace{-2pt}\mathbb{R}^{n}\hspace{-2pt} \rightarrow\hspace{-2pt}\mathbb{R}$ is a continuous convex function possibly non-smooth. Problem \eqref{l1-norm} is obviously a special case of \eqref{l1-general} with $g(\bx)\hspace{-2pt}=\hspace{-2pt}\lambda ||\bx||_{1}$, where $\lambda\hspace{-2pt}>\hspace{-2pt}0$ is a regularization parameter, and $f(\bx)\hspace{-2pt}=\hspace{-2pt}\frac{1}{2}||\mathbf{\Phi}\bx\hspace{-2pt}-\hspace{-2pt}\by||^{2}_{2}$. %, here $L_{\nabla f}=\mathrm{eigen}_{\max}(\mathbf{\Phi}^{\mathrm{T}}\mathbf{\Phi})$ \cite{Beck09}, where $\mathrm{eigen}_{\max}(.)$ is the maximum eigenvalue and $\mathbf{\Phi}^{\mathrm{T}}$ is the transpose of $\mathbf{\Phi}$.
The results of using proximal gradient methods \cite{Beck09} give that $\bx^{(k)}$ at iteration $k$ can be iteratively computed by:
%\vspace{-10pt}
\begin{equation}\label{l1-proximal}
    \bx^{(k)}= \Gamma_{\frac{1}{L}g}\Big(\bx^{(k-1)}\hspace{-2pt}-\hspace{-2pt}\frac{1}{L}\nabla f(\bx^{(k-1)})\Big),
    %\vspace{-7pt}
\end{equation}
where $L\hspace{-2pt}\geq \hspace{-2pt}L_{\nabla f}$ and $\Gamma_{\hspace{-2pt}\frac{1}{L}g}(\bx)$ is a proximal operator that is defined by
%\vspace{-8pt}
\begin{equation}\label{l1-proximalOperator}
    \Gamma_{\frac{1}{L}g}(\bx) = \argmin_{\bv \in\mathbb{R}^{n}}\Big\{ \frac{1}{L}g(\bv) + \frac{1}{2}||\bv-\bx||^{2}_{2}\Big\}.
    %\vspace{-3pt}
\end{equation}
%\vspace{-10pt}

\textbf{CS with SI}. CS with prior information or SI via $\ell_{1}$-$\ell_{1}$ minimization is improved dramatically if SI has good enough quality \cite{MotaGLOBALSIP14,MotaARXIV14,MotaICASSP15}. The $\ell_{1}$-$\ell_{1}$ minimization considers reconstructing $\bx$ given a SI, $\bz \in \mathbb{R}^{n}$ by solving the problem \eqref{l1-general} with $g(\bx)\hspace{-2pt}=\hspace{-2pt}\lambda (||\bx||_{1}+||\bx-\bz||_{1})$, i.e., solving:
%\vspace{-5pt}
\begin{equation}\label{l1-l1minimization}
    \min_{\bx}\Big\{\frac{1}{2}||\mathbf{\Phi}\bx-\by||^{2}_{2} + \lambda (||\bx||_{1}+||\bx-\bz||_{1})\Big\}.
    %\vspace{-5pt}
\end{equation}
The $\ell_{1}$-$\ell_{1}$ minimization \eqref{l1-l1minimization} has an expression for the bound \cite{MotaGLOBALSIP14,MotaARXIV14,MotaICASSP15} on the number of measurements required to successfully reconstruct $\bx$ that is a function \cite{MotaGLOBALSIP14,MotaARXIV14,MotaICASSP15} of the quality of SI $\bz$.

%\vspace{-10pt}
\section{Reconstruction with Multiple SI using Adaptive Weights}
%\vspace{-10pt}
\label{RAMSIA}
%\vspace{-7pt}
\subsection{Problem Statement}\label{problem}
%\vspace{-5pt}
We consider a problem statement regarding how to efficiently reconstruct a sparse source given multiple SIs. The reconstruction for the distributed sensing of sparse sources is motivated by an idea of how multiple heterogeneous-dynamic lightweight cameras can perform robustly on the desired applications under resource-time constraints, where sensing and processing are expensive. CS \cite{DonohoCOM06,CandesTIT06,Beck09} is emerged as an elegant technique to deal with these challenges for such real-time sensing and processing. Despite prohibiting the communication between tiny cameras, we can take advantage of their correlations at the decoder through supported SIs. The known SIs can be some other reconstructed sources existing at the decoder, where we can exploit inter-source redundancies for a robust reconstruction from the reduced sparse source. Therefore, we need an advanced reconstruction flexibly adapting to on-the-fly changes according to the heterogeneous sparse sources.

%\vspace{-5pt}
\begin{figure*}[t!]
\centering
\setlength{\tabcolsep}{1pt}
\renewcommand{\arraystretch}{0.1}
%\subfigure[View $\bx$]{\includegraphics[width=0.15\textwidth]{figures/viewM0.pdf}}\hspace{30pt}
%\subfigure[View $\bz_{1}$]{\includegraphics[width=0.15\textwidth]{figures/viewM1.pdf}} \hspace{30pt}
%\subfigure[View $\bz_{2}$]{\includegraphics[width=0.15\textwidth]{figures/viewM2.pdf}} \\

\subfigure[\vspace{-9pt}View 1]{\includegraphics[width=0.15\textwidth]{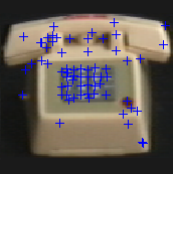}\label{figView1}}
%\vspace{-5pt}
\subfigure[\vspace{-9pt}1000-D histogram $\bx$ of View 1]{\includegraphics[width=0.55\textwidth]{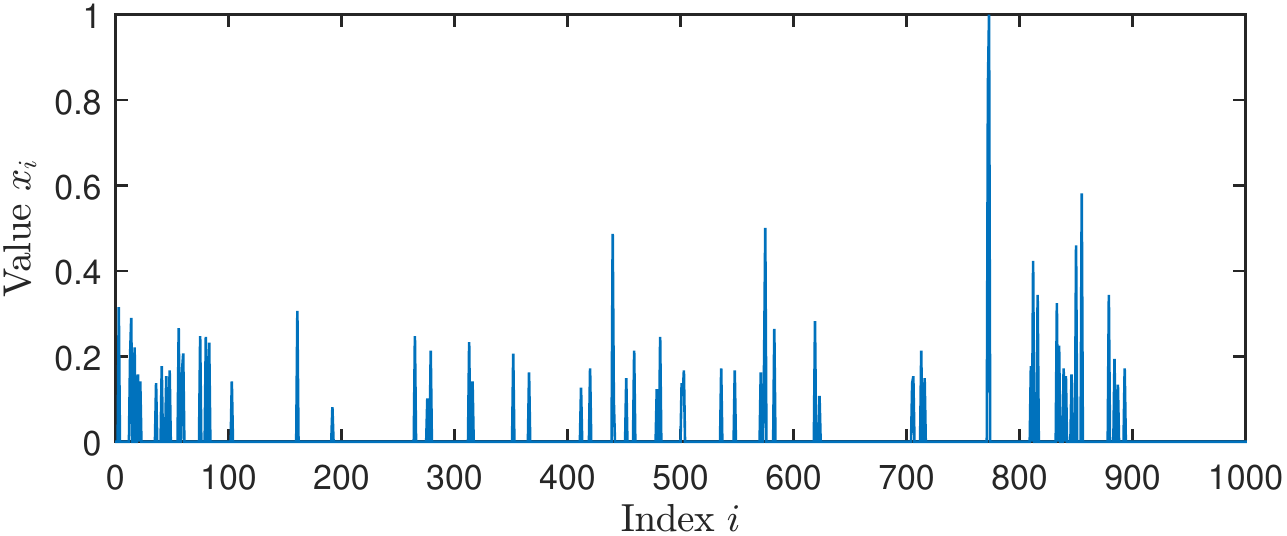}\label{figFeatVecX}}
%\subfigure[250-D $\by$]{\includegraphics[width=0.41\textwidth]{figures/y.pdf}\label{figFeatVecY}}
\\
%\vspace{-5pt}
\subfigure[\vspace{-9pt}View 2]{\includegraphics[width=0.15\textwidth]{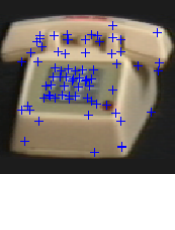}\label{figView2}}
%\vspace{-5pt}
\subfigure[\vspace{-9pt}1000-D histogram $\bz_{1}$ of View 2]{\includegraphics[width=0.55\textwidth]{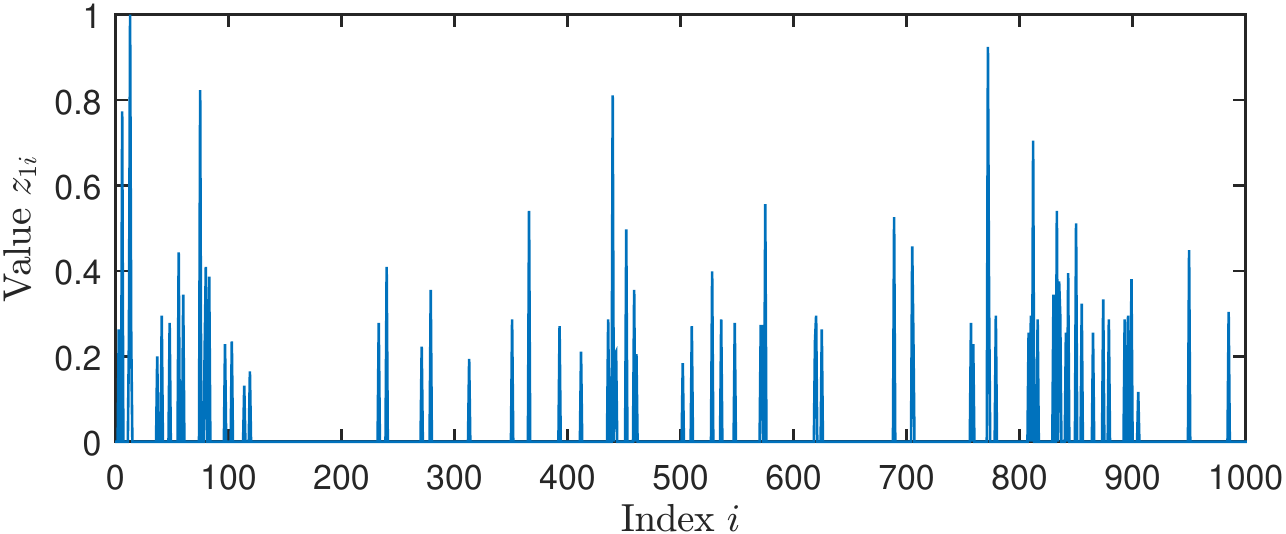}\label{figFeatVecZ1}}\\
%\vspace{-5pt}
\subfigure[\vspace{-9pt}View 3]{\includegraphics[width=0.15\textwidth]{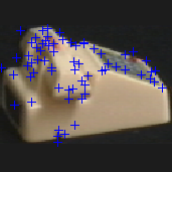}\label{figView3}}
%\vspace{-5pt}
\subfigure[\vspace{-9pt}1000-D histogram $\bz_{2}$ of View 3]{\includegraphics[width=0.55\textwidth]{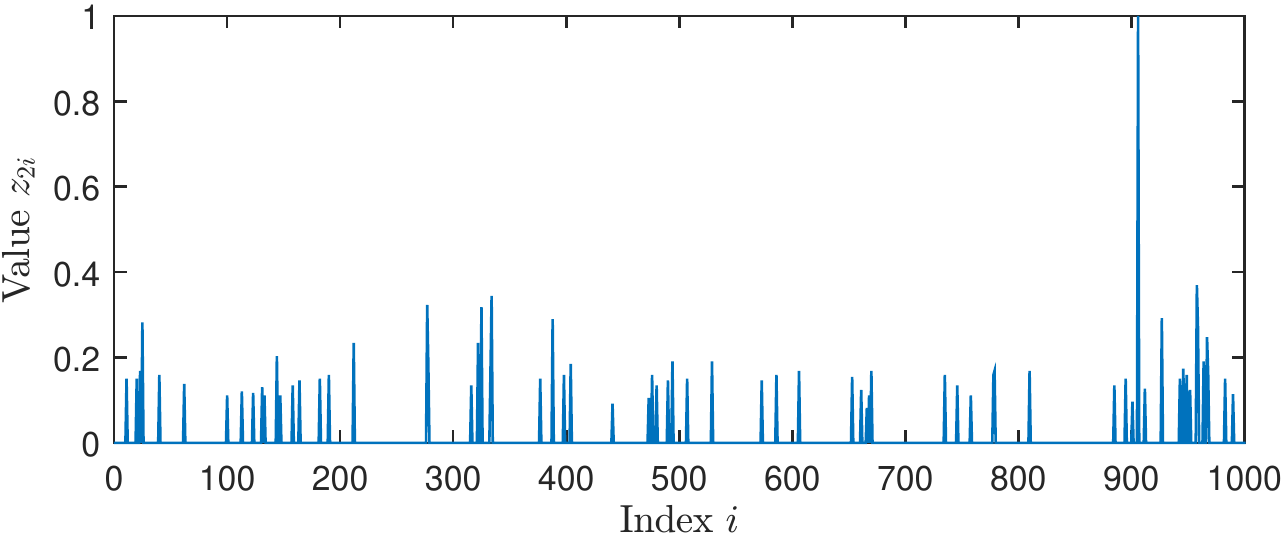}\label{figFeatVecZ2}}
%\subfigure[PSNR 21.2]{\includegraphics[width=0.16\textwidth]{Images/interpolatedffbf}}
%\vspace{-5pt}
\caption{SIFT-feature histograms of object 60 in the COIL-100 \cite{Nene96}.}\label{figFeatVec}
%\vspace{-5pt}
\end{figure*}
%\vspace{-5pt}
Let us consider the example scenario of the distributed compressed sensing of tiny cameras for multiview object recognition in Fig. \ref{figFeatVec}, where a corresponding feature histogram acquired from a given camera is considered as a sparse source $\bx$ (Fig. \ref{figFeatVecX}). Figure \ref{figFeatVec} shows three-view images, View 1 (Fig. \ref{figView1}), View 2 (Fig. \ref{figView2}), View 3 (Fig. \ref{figView3}), of the given object 60 in the COIL-100 database \cite{Nene96} with corresponding SIFT-feature \cite{DLowe99} points to create feature histograms as sparse sources. The feature histograms are created by extracting all SIFT \cite{DLowe99} features from the image then propagating down a hierarchical vocabulary tree based on a hierarchical $k$-means algorithm \cite{NisterCVPR06}. In reality, we may need very high-dimensional histograms, it is essential to reduce the source dimension by CS before further processing. The idea of CS is reducing the source without prior knowledge of the source distribution. Thus a reduced $\by$, which is obtained by compressing 1000 dimensions (D) $\bx$ (Fig. \ref{figFeatVecX}), is to be conveyed to the joint decoder. We use the available compressed $\by$ to reconstruct $\bx$ with given known SIs, 1000-D $\bz_{1}$ (Fig. \ref{figFeatVecZ1})and 1000-D $\bz_{2}$ (Fig. \ref{figFeatVecZ2}), which are feature histograms of neighbor views. Obviously, $\bx,\bz_{1},\bz_{2}$ are naturally correlated in some degree.

In order to reconstruct $\bx$ from $\by$ given $\bz_{1},\bz_{2}$, we may go straight to the solution \eqref{l1-l1minimization} $\ell_{1}$-$\ell_{1}$ minimization with only one input SI either $\bz_{1}$ or $\bz_{2}$. The measurements of $\by$ required to successfully reconstruct $\bx$ is a function of the quality of $\bz_{1},\bz_{2}$. There may be a chance that the $\ell_{1}$-$\ell_{1}$ reconstruction performs worse than the $\ell_{1}$ reconstruction, incurred by a not good enough SI, e.g., $\bz_{2}$,
%We consider the bounds $m_{1}=193$, $m_{2}=269$, on measurements obtained by \eqref{l1-l1 bound} with $\bz_{1},\bz_{2}$, respectively. Moreover, the bound, $m_{0}=247$ without SI, is calculated from \eqref{l1 bound} without the aid of SIs. It is worthily noted that, practically, they are unknown since they are calculated from the unknown $\bx$ at the decoder. Undoubtedly, there may be a chance that the $\ell_{1}$-$\ell_{1}$ reconstruction performs worse than the $\ell_{1}$ reconstruction, e.g., $m_{0}<m_{2}$, incurred by using a not good enough SI $\bz_{2}$. In contrast, the SI $\bz_{1}$ is meaningfully helping the reconstruction with $m_{0}>m_{1}$.
exposing the drawback of the $\ell_{1}$-$\ell_{1}$ reconstruction. %We are interested in how to create a more general and advanced reconstruction able to not only deal with the unexpected negative SI cases but cooperatively take advantage of the multiple supported SIs.
Therefore, we come up with a new Reconstruction Algorithm with Multiple Side Information using Adaptive weights (RAMSIA) uses the measurement $\by$, which is obtained by $\by \hspace{-2pt}= \hspace{-2pt}\mathbf{\Phi} \bx$, and $J$ known SIs, $\bz_{1},\bz_{2},...,\bz_{J} \hspace{-2pt}\in  \hspace{-2pt}\mathbb{R}^{n}$ as inputs. The final goal is to reconstruct output $\bhx$ as best as we can.
%\begin{figure}[!hbt]
%  \centering
%  \includegraphics[width=0.5\textwidth]{figures/RAMSIA.pdf}
%  \vspace{-10pt}
%  \caption{The proposed RAMSIA.}
%  \vspace{-10pt}
%  \label{figDICOSS}
%\end{figure}

RAMSIA is built up using the advances of the $\ell_{1}$-$\ell_{1}$ minimization \eqref{l1-l1minimization} and additionally adaptive weights to robustly work on multiple SIs. The objective function of RAMSIA shall be created as an $n$-$\ell_{1}$ minimization based on the problem in \eqref{l1-general} with
%\vspace{-9pt}
\begin{equation}\label{n-l1g}
%\begin{split}
g(\bx) \hspace{-2pt}= \lambda \Big(\beta_{0}||\mathbf{W}_{0}\bx||_{1} \hspace{-2pt}+ \hspace{-2pt}\sum\limits_{j=1}^{J}\hspace{-2pt}\beta_{j}||\mathbf{W}_{j}(\bx\hspace{-2pt}-\hspace{-2pt}\bz_{j})||_{1}\Big),
    %\vspace{-8pt}
%\end{split}
\end{equation}
%\vspace{-4pt}
where $\beta_{j}\hspace{-2pt}>\hspace{-2pt}0$ are weights among SIs and $\mathbf{W}_{j}$ are a diagonal weight matrix for each SI, $\mathbf{W}_{j}\hspace{-2pt}=\hspace{-2pt}\mathrm{diag}(w_{j1},w_{j2},...,w_{jn})$, wherein $w_{ji}\hspace{-2pt}>\hspace{-2pt}0$ is the weight in $\mathbf{W}_{j}$ at index $i$ for a given SI $\bz_{j}$. To conform the first term $\beta_{0}||\mathbf{W}_{0}\bx||_{1}$ to the remaining terms, we rewrite it as $\beta_{0}||\mathbf{W}_{0}(\bx\hspace{-2pt}-\hspace{-2pt}\bz_{0})||_{1}$ where $\bz_{0}\hspace{-2pt}=\hspace{-2pt}\mathbf{0}$, thus $g(\bx)\hspace{-2pt}=\hspace{-2pt}\lambda \hspace{-2pt}\sum_{j=0}^{J}\hspace{-2pt}\beta_{j}||\mathbf{W}_{j}(\bx\hspace{-2pt}-\hspace{-2pt}\bz_{j})||_{1}$ from \eqref{n-l1g}. We may emphasize as a difference compared with our previous work \cite{LuongDCC16} that we here compute weights in two levels. First $w_{ji}$ for intra-SI weights and then inter-SI weights $\beta_{j}$ among SIs as global weights are computed to optimize the weights. We end up formulating the objective function of the $n$-$\ell_{1}$ problem by:
%\vspace{-10pt}
\begin{equation}\label{n-l1minimizationGlobal}
    \min_{\bx}\hspace{-2pt}\Big\{\hspace{-2pt}H(\bx)\hspace{-2pt}=\hspace{-2pt}\frac{1}{2}||\mathbf{\Phi}\bx-\by||^{2}_{2} + \lambda \hspace{-2pt}\sum\limits_{j=0}^{J}\hspace{-2pt}\beta_{j}||\mathbf{W}_{j}(\bx-\bz_{j})||_{1}\Big\}.
    %\vspace{-8pt}
\end{equation}
It can be noted that this \eqref{n-l1minimizationGlobal} would become the problem \eqref{l1-l1minimization} if $\beta_{j}\hspace{-2pt}=\hspace{-2pt}1$, $\mathbf{W}_{0}\hspace{-2pt}=\hspace{-2pt}\mathbf{W}_{1}\hspace{-2pt}=\hspace{-2pt}\mathbf{I}$, and $\mathbf{W}_{j}\hspace{-2pt}=\hspace{-2pt}\mathbf{0}~(j\hspace{-2pt}\geq \hspace{-2pt}2)$, where $\mathbf{I}$ is a unit diagonal matrix (size of $n\hspace{-2pt}\times \hspace{-2pt}n$), $\mathbf{I}\hspace{-2pt}=\hspace{-2pt}\mathrm{diag}(1,1,...,1)$.
%\vspace{-25pt}
%\subsection{The Reconstruction Algorithm with Multiple Side Information (RAMSIA)}
%\vspace{-10pt}
%\label{algorithmRAMSIA}
\subsection{The proposed RAMSIA}
%\vspace{-5pt}
\label{algorithmRAMSIA}
%\subsection{Adaptive-Weighted Strategies for the $n$-$\ell_{1}$ minimization}
%\vspace{-7pt}
%\label{adaptive-weightedStrategies}
Before solving the $n$-$\ell_{1}$ problem in \eqref{n-l1minimizationGlobal}, we need to address an arisen question how to determine the weight values to improve the reconstruction as well as avoid declination by the negative SIs to the reconstruction. We should distribute relevant weights not only in the intra-SI but also inter SIs. To optimize the objective function in \eqref{n-l1minimizationGlobal}, we impose constraints on both all $\beta_{j}$ and $\mathbf{W}_{j}$, by this way, we will be able to assign weights on multiple SIs according to their qualities during the iterative process. We propose to solve the $n$-$\ell_{1}$ problem in \eqref{n-l1minimizationGlobal} based on the proximal gradient method \cite{Beck09}, i.e., at every iteration $k$ we need to update, on the one hand, the weights $\beta_{j},\mathbf{W}_{j}$ and on the other hand compute $\bx$. We may have different strategies to update the weights depending on our constraints.

To solve the minimization problem in \eqref{n-l1minimizationGlobal}, we propose SI constraints by $\sum_{i=1}^{n}w_{ji}\hspace{-2pt}=\hspace{-2pt}n$ at the intra SI level and $\sum_{j=0}^{J}\beta_{j}\hspace{-2pt}=\hspace{-2pt}1$ at the inter SI level. We minimize the objective function $H(\bx)$ in \eqref{n-l1minimizationGlobal} in two steps, one for $w_{ji}$ and the remain for $\beta_{j}$ during the iterative reconstruction. Referring to our previous work \cite{LuongDCC16}, a different weighting strategy was used, where $\sum_{j=0}^{J}w_{ji}\hspace{-2pt}=\hspace{-2pt}1$ at a given index $i$ cross all SIs and all $\beta_{j}\hspace{-2pt}=\hspace{-2pt}1$.

Firstly, by considering $\bx$ and $\beta_{j}$ fixed, for each SI $\bz_{j}$ we compute $\mathbf{W}_{j}$ by optimizing \eqref{n-l1minimizationGlobal}, i.e., $\arg\hspace{-2pt}\min_{\mathbf{W}_{j}}\{H(\bx)\} \hspace{-2pt}$ is equal to
  %\vspace{-10pt}
\begin{equation}\label{n-l1-weight-minimizationIntra}
%\begin{split}
\arg\hspace{-2pt}\min_{\hspace{-2pt}\mathbf{W}_{j}}\hspace{-2pt}\Big\{ \hspace{-2pt}\lambda \hspace{-4pt}\sum\limits_{j=0}^{J} \hspace{-2pt}\beta_{j}\hspace{-1pt}||\hspace{-1pt}\mathbf{W}_{j}\hspace{-1pt}(\bx \hspace{-2pt}- \hspace{-2pt}\bz_{j})||_{1}\hspace{-3pt}\Big\} \hspace{-2pt} =  \hspace{-2pt} \arg\hspace{-4pt}\min_{\{w_{ji}\}} \hspace{-3pt}\Big\{ \hspace{-2pt}\lambda  \hspace{-1pt}\beta_{j} \hspace{-4pt}\sum\limits_{i=0}^{n}\hspace{-3pt}w_{ji}\hspace{-1pt}|x_{i} \hspace{-2pt}- \hspace{-2pt}z_{ji}| \hspace{-2pt}\Big\}\hspace{-1pt},
 %\vspace{-8pt}
%\end{split}
\end{equation}
where $z_{ji}$ is an element of $\bz_{j}$ at index $i$. We can achieve the
minimization of \eqref{n-l1-weight-minimizationIntra} (following Cauchy inequality) when
all $w_{ji}|x_{i}\hspace{-2pt}-\hspace{-2pt}z_{ji}|$ are equal to a positive parameter $\eta$, i.e., $w_{ji}\hspace{-2pt}=\hspace{-2pt}\eta/(|x_{i}\hspace{-2pt}-\hspace{-2pt}z_{ji}|\hspace{-2pt}+\hspace{-2pt}\epsilon)$ with a small added $\epsilon$. Combining with the intra SI constraint $\sum_{i=1}^{n}w_{ji}\hspace{-2pt}=\hspace{-2pt}n$, we get intra SI weights $w_{ji}$ by:
%  \vspace{-5pt}
%\begin{equation}\label{n-l1-weightsIntra}
%    w_{ji} = \frac{n}{1+(|x_{i}-z_{ji}|+\epsilon)(\sum\limits_{l=1,\neq i}^{n}(|x_{i}-z_{jl}|+\epsilon)^{-1})}.
%    \vspace{-5pt}
%\end{equation}
  %\vspace{-10pt}
\begin{equation}\label{n-l1-weightsIntra}
    w_{ji} \hspace{-2pt}=n\Big/\Big(1\hspace{-2pt}+\hspace{-2pt}(|x_{i}\hspace{-2pt}-\hspace{-2pt}z_{ji}|\hspace{-2pt}+\hspace{-1pt}\epsilon)(\hspace{-2pt}\sum\limits_{l=1,\neq i}^{n}(|x_{i}\hspace{-2pt}-\hspace{-2pt}z_{jl}|\hspace{-2pt}+\hspace{-1pt}\epsilon)^{-1})\Big).
    %\vspace{-7pt}
\end{equation}

Secondly, given $\bx$ and $\mathbf{W}_{j}$, we compute $\beta_{j}$ at the inter-SI level by optimizing the problem in \eqref{n-l1minimizationGlobal}:
  %\vspace{-10pt}
\begin{equation}\label{n-l1-weight-minimizationInter}
%\begin{split}
  \arg\hspace{-2pt}\min_{\{\beta_{j}\}}\{H(\bx)\}= \arg\hspace{-2pt}\min_{\{\beta_{j}\}}\{\lambda \sum\limits_{j=0}^{J}\beta_{j}||\mathbf{W}_{j}(\bx-\bz_{j})||_{1}\}.
  %\vspace{-7pt}
%\end{split}
\end{equation}
Similar to \eqref{n-l1-weight-minimizationIntra}, from \eqref{n-l1-weight-minimizationInter} we obtain $\beta_{j}$ as this formula:
  %\vspace{-8pt}
 \begin{equation}\label{n-l1-Beta}
    \beta_{j}= \eta /({||\mathbf{W}_{j}(\bx-\bz_{j})||_{1}+\epsilon}).
  %\vspace{-7pt}
\end{equation}
%  \vspace{-5pt}
% \begin{equation}\label{n-l1-Beta}
%    \beta_{j}= \frac{\eta }{||\mathbf{W}_{j}(\bx-\bz_{j})||_{1}+\epsilon}.
%  \vspace{-5pt}
%\end{equation}
Combining \eqref{n-l1-Beta} with the inter SI constraint $\sum_{j=0}^{J}\beta_{j}=1$, we get the inter SI weights $\beta_{j}$ by:
  %\vspace{-10pt}
\begin{equation}\label{n-l1-BetaInter}
\beta_{j}\hspace{-2pt} =\hspace{-2pt} 1\hspace{-2pt}\Big/\hspace{-2pt}\Big(1\hspace{-2pt}+\hspace{-1pt}(||\mathbf{W}_{j}(\bx\hspace{-2pt}-\hspace{-2pt}\bz_{j})||_{1}\hspace{-2pt}+\hspace{-1pt}\epsilon)(\hspace{-5pt}\sum\limits_{l=0,\neq j}^{J}\hspace{-2pt}(||\mathbf{W}_{l}(\bx\hspace{-2pt}-\hspace{-2pt}\bz_{l})||_{1}\hspace{-2pt}+\hspace{-1pt}\epsilon)^{-1})\hspace{-2pt}\Big).
  %\vspace{-7pt}
\end{equation}
  %\vspace{-5pt}

Given $\mathbf{W}_{j}$ and $\beta_{j}$, we compute $\bx^{(k)}$ at iteration $k$ to minimize the problem in \eqref{n-l1minimizationGlobal}. %It can be worthily noted that computing $\bx^{(k)}$ for the intra-inter SI weighted scheme is more general than that of the cross SI weighted scheme. Hence, we perform the computation for the intra-inter SI weighted scheme then a similar result to be derived for the cross SI weighted scheme.
Based on the proximal gradient method \cite{Beck09}, $\bx^{(k)}$ is obtained from \eqref{l1-proximal}, where $g(\bx)\hspace{-5pt}=\hspace{-5pt}\lambda \hspace{-2pt}\sum_{j=0}^{J}\hspace{-2pt}\beta_{j}||\mathbf{W}_{j}(\bx\hspace{-2pt}-\hspace{-2pt}\bz_{j})||_{1}$. The remaining key question is how to compute the proximal operator $\Gamma_{\hspace{-2pt}\frac{1}{L}g}(\hspace{-1pt}\bx\hspace{-1pt})$. From \eqref{n-l1-proximalOperatorElementCompute} (derived in Appendix), $\Gamma_{\hspace{-2pt}\frac{1}{L}g}(\hspace{-1pt}x_{i}\hspace{-1pt})$ is obtained by:
%\vspace{-9pt}
\begin{equation}\label{n-l1-proximalOperatorElementFinalInter}
    \Gamma_{\hspace{-2pt}\frac{1}{L}g}(x_{i}) \hspace{-2pt}= \hspace{-5pt}\left\{\hspace{-2pt}
  \begin{array}{l}
    \hspace{-5pt}x_{i}\hspace{-2pt}-\hspace{-2pt}\frac{\lambda}{L} \hspace{-2pt} \sum\limits_{j=0}^{J}\hspace{-2pt}\beta_{j}w_{ji}(\hspace{-2pt}-1\hspace{-2pt})^{\mathfrak{b}(l<j)}\mathrm{~if\hspace{1pt}}\eqref{n-l1-proximalX}
; \\[-3pt]
 \hspace{-5pt}z_{li}~~~~~~~~~~~~~~~~~~~~~~~~~~~~~~~~~~~~~~\mathrm{if\hspace{1pt}} \eqref{n-l1-proximalZ};
  \end{array}
\right.
%\vspace{-4pt}
\end{equation}

Finally, we sum up the proposed RAMSIA in Algorithm \ref{RAMSIAAlg-IC} based on a fast iterative FISTA algorithm \cite{Beck09}, where the main difference from FISTA is the iterative computation of the updated weights. RAMSIA iteratively updates two weight levels in turn $w_{ji}$ in \eqref{n-l1-weightsIntra} and $\beta_{j}$ in \eqref{n-l1-BetaInter}. It can be noted that the \textit{Stopping criteria} in Algorithm \ref{RAMSIAAlg-IC} can be either a maximum iteration number $k_{\max}$, a relative variation of the objective function $H(\bx)$ \eqref{n-l1minimizationGlobal}, or a change of the number of non-zero components of the estimate $\bx^{(k)}$. In this work, the relative variation of $H(\bx)$ \eqref{n-l1minimizationGlobal} is chosen.
\setlength{\textfloatsep}{0pt}% Remove \textfloatsep
\begin{algorithm}[t!]
%  \algsetup{linenosize=\tiny}
%  \scriptsize
\DontPrintSemicolon \SetAlgoLined
\textbf{Input}: $\by,\mathbf{\Phi},\bz_{1},\bz_{2},...,\bz_{J}$;\\
\textbf{Output}: $\bhx$;\\
%\KwData{$\by,\mathbf{\Phi},\bz_{1},\bz_{2},...,\bz_{J}$.}%\KwResult{$\bhx$.}
%\vspace{-2pt}
  \tcp{Initialization.}
  %\vspace{-2pt}
    $\mathbf{W}_{0}^{(1)}\hspace{-2pt}=\hspace{-2pt}\mathbf{I}$; $\beta_{0}^{(1)}\hspace{-2pt}=\hspace{-2pt}1$; $\mathbf{W}_{j}^{(1)}\hspace{-2pt}=\hspace{-2pt}\mathbf{0}$;  $\beta_{j}^{(1)}\hspace{-2pt}=\hspace{-2pt}0~(1\hspace{-2pt}\leq \hspace{-2pt} j\hspace{-2pt}\leq \hspace{-2pt}J)$;
    %\vspace{-1pt}
    %~(1\leq j\leq J)$; $\bz_{-1}=-\bi$; $\bz_{0}=\mathbf{0}$; $\bz_{J+1}=\bi$;\\
% $b(l<j)=1$ if $l<j$ otherwise $b(l<j)=0$ ($-1\leq l \leq J$);\\
   $\bu^{(1)}\hspace{-2pt}=\hspace{-2pt}\bx^{(0)}\hspace{-2pt}=\hspace{-2pt}\mathbf{0}$; $L\hspace{-2pt}=\hspace{-2pt}L_{\nabla f}$; $\lambda,\epsilon \hspace{-2pt}>\hspace{-2pt}0$; $t_{1}\hspace{-2pt}=\hspace{-2pt}1$; $k\hspace{-2pt}=\hspace{-2pt}0$; \\
   %\vspace{-1pt}
  \While{Stopping criterion is false}{
       $k=k+1$; \\
       %\vspace{-4pt}
       \tcp{Solving given the weights.}
       %\vspace{-2pt}
        $\nabla f(\bu^{(k)})=\mathbf{\Phi}^{\mathrm{T}}(\mathbf{\Phi} \bu^{(k)}-\by)$; \\
	   $\bx^{(k)}\hspace{-2pt}= \hspace{-2pt}\Gamma_{\hspace{-2pt}\frac{1}{L}g}\hspace{-2pt}\Big(\bu^{(k)}\hspace{-2pt}-\hspace{-2pt}\frac{1}{L}\nabla f(\bu^{(k)})\hspace{-2pt}\Big)$; $\Gamma_{\hspace{-2pt}\frac{1}{L}g}(.)$ is given by \eqref{n-l1-proximalOperatorElementFinalInter};\\
%\vspace{-2pt}
	   \tcp{Computing the updated weights.}
%\vspace{-2pt}
%	   $w_{ji}^{(k+1)} = n\Big/\Big(1+(|x_{i}^{(k)}-z_{ji}|+\epsilon)(\sum\limits_{l=1,\neq i}^{n}(|x_{i}^{(k)}-z_{jl}|+\epsilon)^{-1})\Big)$;
	   $w_{ji}^{(k+1)} = \frac{n}{1+\Big(|x_{i}^{(k)}-z_{ji}|+\epsilon\Big)\Big(\sum\limits_{l\neq i}(|x_{i}^{(k)}-z_{jl}|+\epsilon)^{-1}\Big)}$;
%\vspace{-5pt}
\\
%	   \tcp{Computing the updated inter weights.}
%\vspace{-5pt}
   $\beta_{j}^{(k+1)}\hspace{-2pt} =\hspace{-2pt}$\\ $\frac{1}{\hspace{-1pt}1\hspace{-1pt}+\hspace{-1pt}\Big(||\mathbf{W}_{j}^{(k+1)}(\bx^{(k)}\hspace{-1pt}-\hspace{-1pt}\bz_{j})||_{1}\hspace{-1pt}+\hspace{-1pt}\epsilon\Big)\Big(\hspace{-2pt}\sum\limits_{l\neq j}\hspace{-3pt}(||\mathbf{W}_{l}^{(k+1)}(\bx^{(k)}\hspace{-1pt}-\hspace{-1pt}\bz_{l})||_{1}\hspace{-1pt}+\hspace{-1pt}\epsilon)^{-1}\Big)\hspace{-2pt}}$;
%	   $\beta_{j}^{(k+1)}\hspace{-2pt} =\hspace{-2pt} \frac{1}{1\hspace{-0pt}+\hspace{-0pt}(||\mathbf{W}_{j}^{(k+1)}(\bx^{(k)}\hspace{-0pt}-\hspace{-0pt}\bz_{j})||_{1}+\epsilon)(\hspace{-2pt}\sum\limits_{l=0,\neq j}^{J}\hspace{-2pt}(||\mathbf{W}_{l}^{(k+1)}(\bx^{(k)}\hspace{-0pt}-\hspace{-0pt}\bz_{l})||_{1}+\epsilon)^{-1})}$;
%\vspace{-2pt}
\\
%\vspace{-4pt}
       \tcp{Updating new values.}
       %\vspace{-2pt}
        $t_{k+1}\hspace{-2pt}=\hspace{-2pt}(1\hspace{-2pt}+\hspace{-2pt}\sqrt{1\hspace{-2pt}+\hspace{-2pt}4t_{k}^{2}})/2$;\\
        $\bu^{(k+1)}\hspace{-2pt}=\hspace{-2pt}\bx^{(k)}\hspace{-2pt}+\hspace{-2pt}\frac{t_{k}\hspace{-2pt}-\hspace{-1pt}1}{t_{k+1}}(\bx^{(k)}\hspace{-2pt}-\hspace{-2pt}\bx^{(k-1)})$;\\
       %\vspace{-3pt}
  }
%\vspace{-5pt}
%\textbf{Output}: $\bhx$=$\bx^{(k)}$;
\Return $\bx^{(k)}$;
\caption{The proposed RAMSIA algorithm.}\label{RAMSIAAlg-IC}
%\vspace{-2pt}
\end{algorithm}

%\vspace{-10pt}
\section{Experimental Results}
\label{Experiment}
%\vspace{-10pt}
\begin{figure*}[t!]
  \centering
\subfigure[\vspace{-15pt}Object 16]{\includegraphics[width=0.71\textwidth]{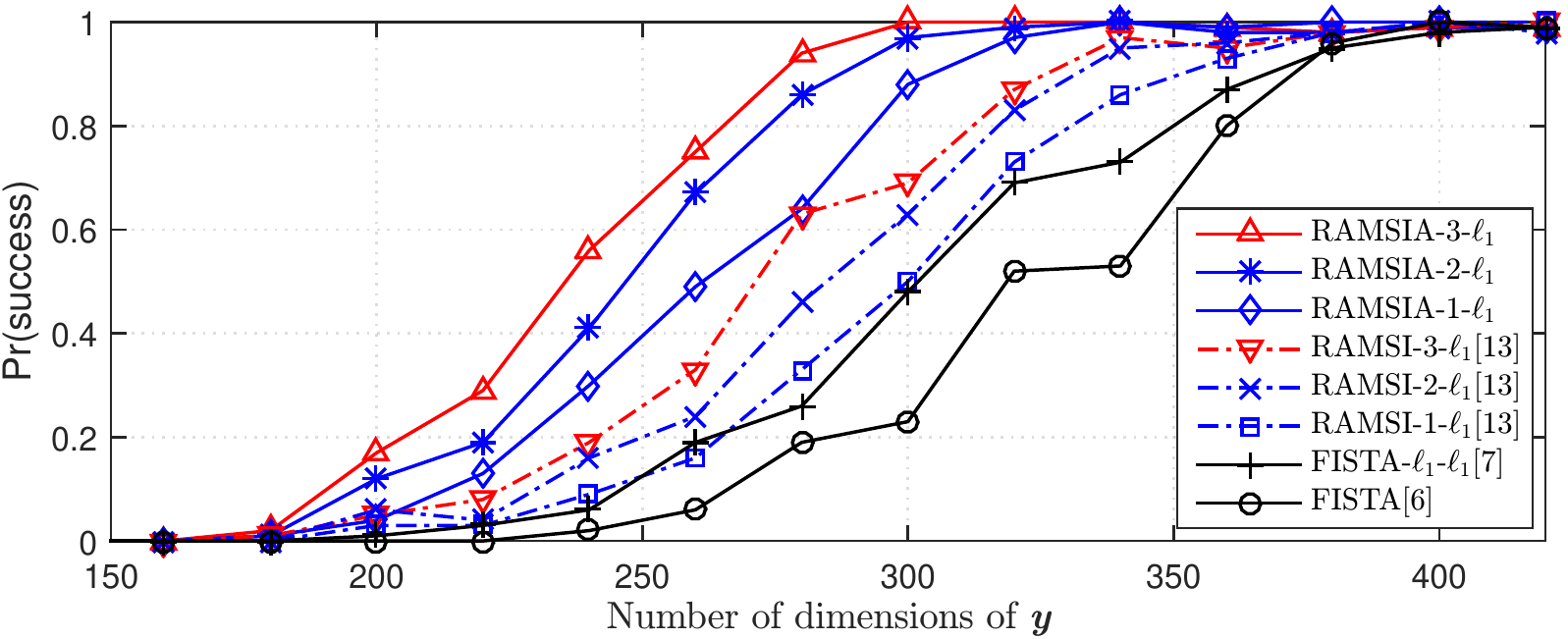}\label{fig59}}\hspace{0pt}
%\vspace{-13pt}
\subfigure[\vspace{-15pt}Object 60]{\includegraphics[width=0.71\textwidth]{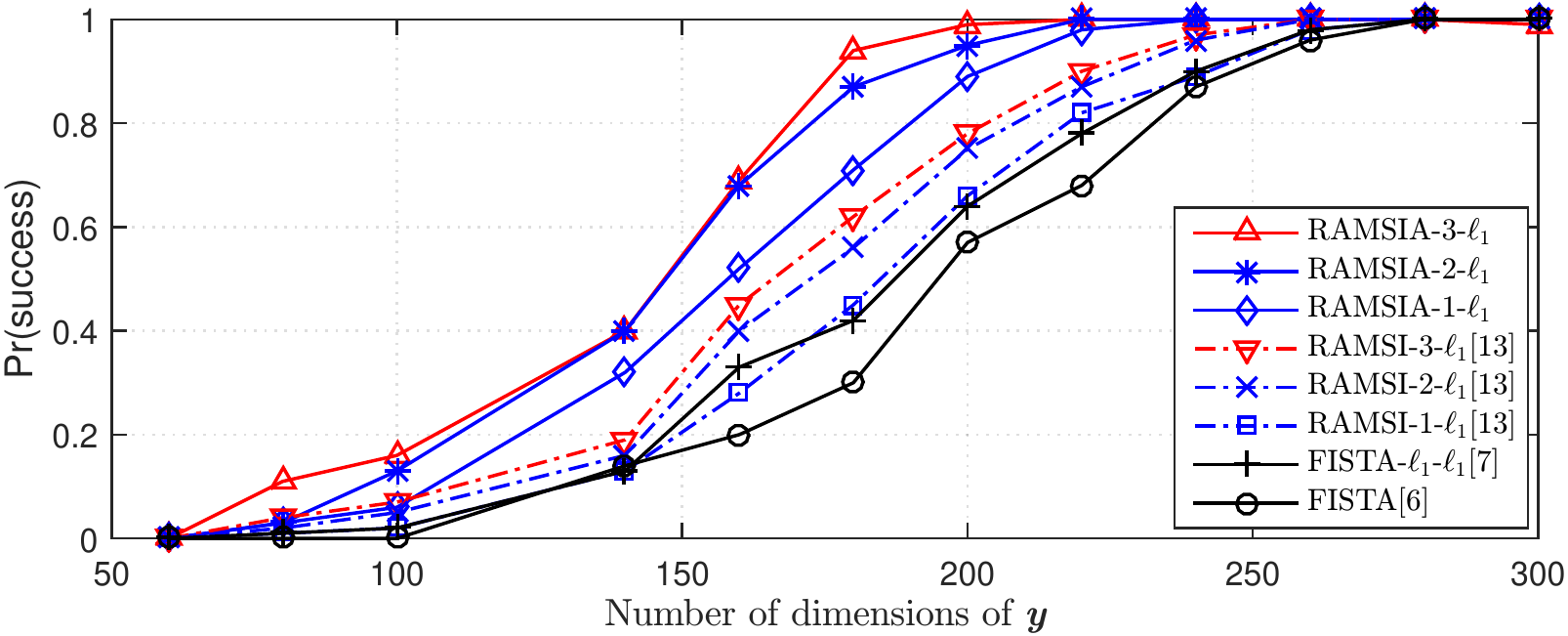}\label{fig60}} \hspace{0pt}
%\vspace{-13pt}
\subfigure[\vspace{-15pt}Generated signals]{\includegraphics[width=0.71\textwidth]{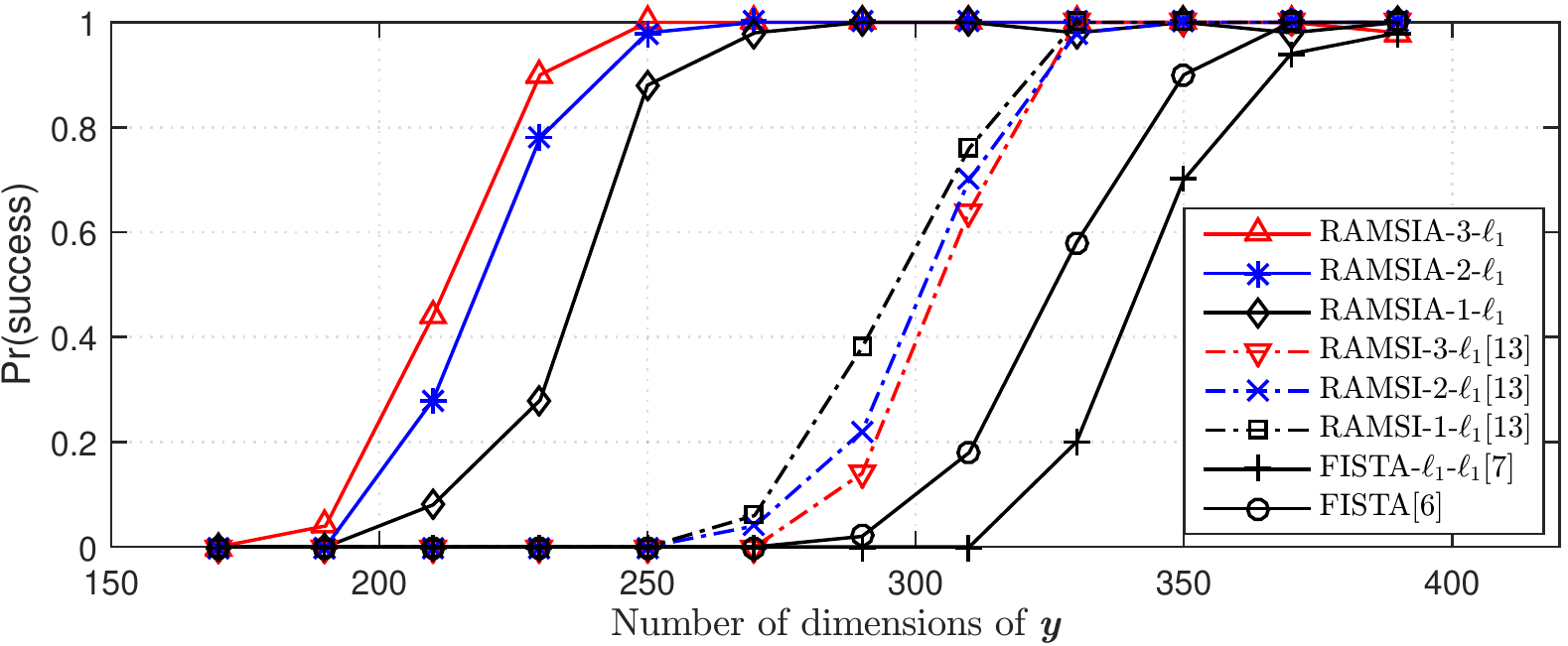}\label{fig100300}} \hspace{0pt}
%  \includegraphics[width=0.7\textwidth]{figures/accuracyRAMSIA12.pdf}
 %\vspace{-13pt}
  \caption{Success probabilities of reconstructing the original 1000-D $\bx$ vs. number of dimensions of the compressed $\by$ using 1,2,3 SIs.}
  %\vspace{5pt}
  \label{figAccuracy}
\end{figure*}
%\subsection{Numerical Experiments}\label{numericalExperiments}
%\vspace{-7pt}
%We present the experiments demonstrating the efficiencies of the RAMSIA algorithms on signals which are generated with different charecteristics.
%%\vspace{-7pt}
%\subsection{Parameter Analysis}\label{parameterAnalysis}
%%\vspace{-7pt}
%The robustness of our reconstruction process may be alternated by choosing the parameter $\epsilon$. In this experiment, we would like to consider how various values of $\epsilon$ influence on the reconstruction robustness.
%\vspace{-7pt}
%\subsection{Compression on Feature Histograms for Multiview Recognition}
%\vspace{-5pt}
%\label{multiviewRecognition}
First we consider the reconstruction of multiview feature histograms as sparse sources , which could be used in multiview object recognition. Given an image, its feature histogram is formed as in Sec. \ref{problem}, e.g., $\bx$. The size of the vocabulary tree \cite{NisterCVPR06} depends on the value of $k$ and number of hierarchies, e.g., if $k\hspace{-2pt} = \hspace{-2pt}10$ and 3 hierarchies, $n \hspace{-2pt}= \hspace{-2pt}1000$ vocabularies as 1000-D, which may be a very big number in reality. Because of the small number of features in a single image, the histogram vector $\bx$ is sparse. Therefore, $\bx$ is first projected into the compressed vector $\by$ that is to be sent to the decoder. At the joint decoder, we suppose $\bx$ is to be reconstructed given some already decoded histograms of neighbor views, e.g., $\bz_{1}$,$\bz_{2}$,$\bz_{3}$. In this work, we use the COIL-100 \cite{Nene96} containing multiview images of 100 small objects with different angle degrees. For the sake of ensuring our experimental setup  a realistic circumstance of multiple distributed sources without given correspondences, we randomly select the 4 neighbor views (3 neighbors as SIs) of representative objects 16 (Fig. \ref{fig59}) and 60 (Fig. \ref{fig60}) over 72 views captured through 360 degrees in the COIL-100 \cite{Nene96} multiview database. In addition, we consider the reconstruction of a generated sparse source $\bx$ given three known SIs (Fig. \ref{fig100300}). We generate $\bx$ with $n\hspace{-2pt}=\hspace{-2pt}1000$ with 100 supports (number of non-zeros), which is generated from the
standard i.i.d. Gaussian distribution. To illustrate a scenario of the less correlated SIs, we generate SIs with all supports of subtractions, $\bx\hspace{-2pt}-\hspace{-2pt}\bz_{1}$, $\bx\hspace{-2pt}-\hspace{-2pt}\bz_{1}$, $\bx\hspace{-2pt}-\hspace{-2pt}\bz_{1}$ are equal to 300.

%\label{performance}
We will evaluate and compare the reconstruction accuracy of the proposed RAMSIA exploiting different number of SIs against the existing $\ell_{1}$ reconstructions. Let RAMSIA-$J$-$\ell_{1}$ denote the RAMSIA reconstruction, where $J$ indicates number of SIs, e.g., RAMSIA-$1$-$\ell_{1}$, RAMSIA-$2$-$\ell_{1}$, RAMSIA-$3$-$\ell_{1}$ are RAMSIA using 1 ($\bz_{1}$), 2 ($\bz_{1}$,$\bz_{2}$), and 3 ($\bz_{1}$,$\bz_{2}$,$\bz_{3}$) SIs, respectively. Similarly, RAMSI-$1$-$\ell_{1}$, RAMSI-$2$-$\ell_{1}$, RAMSI-$3$-$\ell_{1}$ are three versions in our previous work \cite{LuongDCC16}. Let FISTA-$\ell_{1}$-$\ell_{1}$ denote the $\ell_{1}$-$\ell_{1}$ CS reconstruction with one SI ($\bz_{1}$) \cite{MotaGLOBALSIP14}. The existing FISTA \cite{Beck09} and FISTA-$\ell_{1}$-$\ell_{1}$ reconstructions \cite{MotaGLOBALSIP14} are used for comparison. The original source 1000-D $\bx$ is compressed into different lower-dimensions $\by$. We assess the accuracy of a reconstructed $\bhx$ versus the $\bx$ via the success probability, denoted as $\mathrm{Pr(success)}$, versus the number of dimensions. For a fixed dimension, the $\mathrm{Pr(success)}$ is the number of times, in which the source $\bx$ is recovered as $\bhx$ with an error $||\bhx\hspace{-4pt}-\hspace{-4pt}\bx||_{2}/||\bx||_{2}\hspace{-4pt} \leq \hspace{-4pt}10^{-3}$, divided by the total 100 trials (each trial considered randomly selecting 4 neighbor views for objects 16,60 (Figs. \ref{fig59}+\ref{fig60}) and different generated $\bx,\bz_{1}$,$\bz_{2}$,$\bz_{3},\mathbf{\Phi}$ for generated signals (Fig. \ref{fig100300})). RAMSIA (Algorithm \ref{RAMSIAAlg-IC}) thus has $n\hspace{-2pt}=\hspace{-2pt}1000$, $m\hspace{-2pt}<\hspace{-2pt}n$, $J\hspace{-2pt}=\hspace{-2pt}1,2,3$ and we set parameters $\epsilon\hspace{-2pt}=\hspace{-2pt}0.1$, $\lambda\hspace{-2pt}=\hspace{-2pt}10^{-5}$.% and initialize $\bx^{(0)}=\mathbf{\Phi}^{-1}\by$.

Figure \ref{figAccuracy} presents the performances of the proposed RAMSIA as well as RAMSI \cite{LuongDCC16} with 1,2,3 SIs, FISTA \cite{Beck09}, and FISTA-$\ell_{1}$-$\ell_{1}$ \cite{MotaGLOBALSIP14} in terms of the success probabilities versus dimensions. It shows clearly that RAMSIA significantly improves the reconstruction accuracy. Furthermore, the results using 3 SIs, RAMSIA-3-$\ell_{1}$ gives highest accuracy and the performance of RAMSIA-2-$\ell_{1}$ is higher than that of RAMSIA-1-$\ell_{1}$. For instance, RAMSIA requires 300, 320, 340 measurements corresponding to using 3, 2, 1 SIs to successfully reconstruct Object 16 (Fig. \ref{fig59}). Despite exploiting only the same one SI, RAMSIA-1-$\ell_{1}$ is absolutely better than FISTA-$\ell_{1}$-$\ell_{1}$. For the signal in Fig. \ref{fig100300}, the accuracies of FISTA-$\ell_{1}$-$\ell_{1}$ are worse than those of FISTA, i.e., SI does not help, however, RAMSIA-$\ell_{1}$-$\ell_{1}$ still significantly outperforms FISTA. These results may highlight the drawback of $\ell_{1}$-$\ell_{1}$ minimization when SI qualities are not good enough. Therefore, if we exploit the weighted $n$-$\ell_{1}$ minimization as RAMSIA, our reconstruction can deal with the negative instances by adaptive weights, e.g., less weights on such negative cases of SI $\bz_{1}$, to achieve the improvements for all scenarios.

%\vspace{-10pt}
\section{Conclusion}
%\vspace{-7pt}
\label{conclusion}
This paper proposed an sparse signal reconstruction with multiple SIs for the distributed sources by solving the general $n$-$\ell_{1}$ minimization problem. The proposed RAMSIA scheme iteratively updated the adaptive weights in two levels not only in each intra-SI but among SIs to adaptively optimize the reconstruction. RAMSIA took advantage of exploiting both intra-source and inter-source redundancies to adapt to the changes of the heterogeneous sparse sources to deal with the different SI qualities. We experimentally tested RAMSIA on the multiple feature histograms as multiview sparse sources captured from distributed cameras and also the generated signals. The results showed that RAMSIA robustly outperformed the conventional CS and recent $\ell_{1}$-$\ell_{1}$ minimization methods. Moreover, RAMSIA with higher number of SIs gained more improvements than the one with smaller number of SIs.

%\appendices
%\vspace{-10pt}
\section{Appendix}\label{appendix}
%\vspace{-10pt}
We shall compute the proximal operator $\Gamma_{\hspace{-5pt}\frac{1}{L}g}(\bx)$ \eqref{l1-proximalOperator} with $g(\bx)\hspace{-2pt}=\hspace{-2pt}\lambda \hspace{-2pt} \sum_{j=0}^{J}\hspace{-2pt}\beta_{j}\hspace{-1pt}||\mathbf{W}_{j}(\hspace{-1pt}\bx\hspace{-2pt}-\hspace{-2pt}\bz_{j}\hspace{-1pt})||_{1}$. From \eqref{l1-proximalOperator}, $\Gamma_{\frac{1}{L}g}(\bx)$ is expressed by:
%\vspace{-11pt}
\begin{equation}\label{n-l1-proximalOperatorCompute}
    \Gamma_{\hspace{-2pt}\frac{1}{L}g}(\bx)\hspace{-2pt} = \hspace{-2pt}\arg\hspace{-2pt}\min_{\bv \in\mathbb{R}^{n}}\Big\{\hspace{-2pt} \frac{\lambda}{L} \hspace{-2pt}\sum\limits_{j=0}^{J}\hspace{-2pt}\beta_{j}||\mathbf{W}_{j}(\bv\hspace{-2pt}-\hspace{-2pt}\bz_{j})||_{1} \hspace{-2pt}+\hspace{-2pt} \frac{1}{2}||\bv\hspace{-2pt}-\hspace{-2pt}\bx||^{2}_{2}\Big\}.
    %\vspace{-7pt}
\end{equation}
We may note that both terms in \eqref{n-l1-proximalOperatorCompute} are separable in $\bv$ and thus we can minimize each element $v_{i}$ of $\bv$ individually as %The \eqref{n-l1-proximalOperatorCompute} expression of $v_{i}$ is simplified to:
%\vspace{-10pt}
\begin{equation}\label{n-l1-proximalOperatorElement}
    \Gamma_{\hspace{-2pt}\frac{1}{L}g}(x_{i})\hspace{-2pt} =\hspace{-2pt} \argmin_{v_{i} \in\mathbb{R}}\hspace{-3pt}\Big\{\hspace{-2pt}h(v_{i})\hspace{-2pt}=\hspace{-2pt}\frac{\lambda}{L}\hspace{-4pt} \sum\limits_{j=0}^{J}\hspace{-2pt}\beta_{j}w_{ji}|v_{i}\hspace{-2pt}-\hspace{-2pt}z_{ji}|\hspace{-2pt} + \hspace{-2pt} \frac{1}{2}(v_{i}\hspace{-2pt}-\hspace{-2pt}x_{i})^{2}\hspace{-2pt}\Big\}.
    %\vspace{-7pt}
\end{equation}

We consider the $\partial h(v_{i})/\partial v_{i}$. Without loss of generality, we assume $-\infty \hspace{-2pt}\leq \hspace{-2pt}z_{0i}\hspace{-2pt}\leq \hspace{-2pt}z_{1i}\hspace{-2pt}\leq\hspace{-2pt}...\hspace{-2pt}\leq \hspace{-2pt}z_{Ji}\hspace{-2pt}\leq \hspace{-2pt}\infty$. For convenience, let us denote $z_{-1i}\hspace{-2pt}=\hspace{-2pt}-\infty$ and $z_{J+1i}\hspace{-2pt}=\hspace{-2pt}\infty$. When $v_{i}$ is located in one of the intervals, we suppose $v_{i}\hspace{-2pt}\in \hspace{-2pt}(z_{li},z_{l+1i})$ with $-1\hspace{-2pt}\leq \hspace{-2pt}l \hspace{-2pt}\leq \hspace{-2pt}J$, where $\partial h(v_{i})$ exists. Taking the derivative of $h(v_{i})$ in $(z_{li},z_{l+1i})$ delivers
%\vspace{-8pt}
\begin{equation}\label{n-l1-proximalOperatorElementDerivative}
    \frac{\partial h(v_{i})}{\partial v_{i}}\hspace{-2pt}=\hspace{-2pt} \frac{\lambda}{L} \sum\limits_{j=0}^{J}\beta_{j}w_{ji}\mathrm{sign}(v_{i}\hspace{-2pt}-\hspace{-2pt}z_{ji})\hspace{-2pt} + \hspace{-2pt}(v_{i}\hspace{-2pt}-\hspace{-2pt}x_{i}),
    %\vspace{-9pt}
\end{equation}
where $\mathrm{sign}(.)$ is a sign function. In addition, let $\mathfrak{b}(.)$ denote a boolean function, i.e., $\mathfrak{b}(l\hspace{-2pt}<\hspace{-2pt}j)\hspace{-2pt}=\hspace{-2pt}1$ if $l\hspace{-2pt}<\hspace{-2pt}j$, otherwise $\mathfrak{b}(l\hspace{-2pt}<\hspace{-2pt}j)\hspace{-2pt}=\hspace{-2pt}0$. Consequently, $\mathrm{sign}(v_{i}\hspace{-2pt}-\hspace{-2pt}z_{ji})\hspace{-2pt}=\hspace{-2pt}(-1)^{\mathfrak{b}(l<j)}$ and from \eqref{n-l1-proximalOperatorElementDerivative}, we rewrite:
%\vspace{-7pt}
\begin{equation}\label{n-l1-proximalOperatorElementBoolean}
    \frac{\partial h(v_{i})}{\partial v_{i}}\hspace{-2pt}= \hspace{-2pt}\frac{\lambda}{L} \hspace{-2pt}\sum\limits_{j=0}^{J}\beta_{j}w_{ji}(-1)^{\mathfrak{b}(l<j)}\hspace{-2pt} + \hspace{-2pt}(v_{i}\hspace{-2pt}-\hspace{-2pt}x_{i}).
    %\vspace{-3pt}
\end{equation}

When setting $\partial h(v_{i})/\partial v_{i}\hspace{-2pt}=\hspace{-2pt}0$ to minimize $h(v_{i})$, we derive:
%\vspace{-9pt}
\begin{equation}\label{partialZero}
  v_{i}\hspace{-2pt}=\hspace{-2pt}x_{i}\hspace{-2pt}-\hspace{-2pt}\frac{\lambda}{L}\hspace{-2pt} \sum_{j=0}^{J}\hspace{-2pt}\beta_{j}w_{ji}(-1)^{\mathfrak{b}(l<j)}.
  %\vspace{-4pt}
\end{equation}
This $v_{i}$ \eqref{partialZero} is only valid in $(z_{li},z_{l+1i})$, i.e.,%if $z_{li}<x_{i}-\frac{\lambda}{L} \sum_{j=0}^{J}\beta_{j}w_{ji}(-1)^{b(l<j)}<z_{l+1i}$ and is equivalent to:
 %\vspace{-5pt}
 \begin{equation}\label{n-l1-proximalX}
    z_{li}\hspace{-2pt}+\hspace{-2pt}\frac{\lambda}{L}\hspace{-2pt} \sum\limits_{j=0}^{J}\hspace{-2pt}\beta_{j}w_{ji}(\hspace{-2pt}-1\hspace{-2pt})^{\mathfrak{b}(l<j)}\hspace{-4pt}<\hspace{-3pt}x_{i}\hspace{-3pt}<\hspace{-4pt}z_{l+1i}\hspace{-2pt}+\hspace{-2pt}\frac{\lambda}{L}\hspace{-2pt} \sum\limits_{j=0}^{J}\hspace{-2pt}\beta_{j}w_{ji}(\hspace{-2pt}-1\hspace{-2pt})^{\mathfrak{b}(l<j)}.
    %\vspace{-5pt}
\end{equation}
In case of that $x_{i}$ does not belong to alike intervals in \eqref{n-l1-proximalX}, i.e.,
%\vspace{-5pt}
 \begin{equation}\label{n-l1-proximalZ}
z_{li}\hspace{-2pt}+\hspace{-1pt}\frac{\lambda}{L} \hspace{-3pt}\sum\limits_{j=0}^{J}\hspace{-2pt}\beta_{j}w_{ji}(\hspace{-2pt}-1\hspace{-2pt})^{\mathfrak{b}(l-1<j)}\hspace{-4pt}\leq \hspace{-3pt}x_{i}\hspace{-3pt}\leq \hspace{-2pt} z_{li}\hspace{-2pt}+\hspace{-1pt}\frac{\lambda}{L} \hspace{-3pt}\sum\limits_{j=0}^{J}\hspace{-2pt}\beta_{j}w_{ji}(\hspace{-2pt}-1\hspace{-2pt})^{\mathfrak{b}(l<j)}.
    %\vspace{-5pt}
\end{equation}
We will prove that $h(v_{i})$ \eqref{n-l1-proximalOperatorElement} is minimum when $v_{i}=z_{li}$ in the following Lemma \ref{lemmaMin}.
%\vspace{-5pt}
\begin{lemma}
\label{lemmaMin}
Given $x_{i}$ belonging to the intervals represented in \eqref{n-l1-proximalZ}, $h(v_{i})$ in \eqref{n-l1-proximalOperatorElement} is minimum when $v_{i}=z_{li}$.
\end{lemma}
%\vspace{-10pt}
\begin{proof}
We re-express $h(v_{i})$ \eqref{n-l1-proximalOperatorElement} by:
%\vspace{-8pt}
\begin{equation}
\label{n-l1-proximalOperatorElementLemma}
  h(\hspace{-1pt}v_{i}\hspace{-1pt})\hspace{-3pt} = \hspace{-3pt}\frac{\lambda}{L}\hspace{-3pt} \sum_{j=0}^{J}\hspace{-3pt}\beta_{j}w_{ji}|(\hspace{-1pt}v_{i}\hspace{-2pt}-\hspace{-2pt}z_{li}\hspace{-1pt})\hspace{-2pt}-\hspace{-2pt}(\hspace{-1pt}z_{ji}\hspace{-2pt}-\hspace{-2pt}z_{li}\hspace{-1pt})| \hspace{-2pt}+ \hspace{-2pt}\frac{1}{2}((\hspace{-1pt}v_{i}\hspace{-2pt}-\hspace{-2pt}z_{li}\hspace{-1pt})\hspace{-2pt}-\hspace{-2pt}(\hspace{-1pt}x_{i}\hspace{-2pt}-\hspace{-2pt}z_{li}\hspace{-1pt}))^{2}.
%\vspace{-5pt}
\end{equation}
Applying a simple inequality $|a\hspace{-2pt}-\hspace{-2pt}b|\geq |a|\hspace{-2pt}-\hspace{-2pt}|b|$, where $a,b\in \mathbb{R}$, to the first term and expanding the second term in \eqref{n-l1-proximalOperatorElementLemma}, we obtain:
%\begin{eqnarray}
%\label{n-l1-proximalOperatorElementLemmaInequality}
%% \nonumber to remove numbering (before each equation)
%  h(v_{i}) &\geq& \frac{\lambda}{L} \sum\limits_{j=0}^{J}w_{ji}|v_{i}-z_{li}|-\frac{\lambda}{L} \sum\limits_{j=0}^{J}w_{ji}|z_{ji}-z_{li}| \nonumber \\
%   && + \frac{1}{2}(v_{i}-z_{li})^{2}-(v_{i}-z_{li})(x_{i}-z_{li})+\frac{1}{2}(x_{i}-z_{li})^{2}.
%\end{eqnarray}
%\vspace{-8pt}
\begin{equation}
\label{n-l1-proximalOperatorElementLemmaInequality}
\begin{split}
h(v_{i})&\hspace{-2pt} \geq\hspace{-0pt}\frac{\lambda}{L} \hspace{-2pt}\sum\limits_{j=0}^{J}\hspace{-2pt}\beta_{j}w_{ji}|v_{i}\hspace{-2pt}-\hspace{-2pt}z_{li}|\hspace{-2pt}-\hspace{-2pt}\frac{\lambda}{L} \hspace{-2pt}\sum\limits_{j=0}^{J}\hspace{-2pt}\beta_{j}w_{ji}|z_{ji}\hspace{-2pt}-\hspace{-2pt}z_{li}|\hspace{-0pt} +\\[-10pt]
%\vspace{-9pt}
&\hspace{-0pt}\frac{1}{2}(v_{i}\hspace{-2pt}-\hspace{-2pt}z_{li})^{2}\hspace{-2pt}-\hspace{-2pt}(v_{i}\hspace{-2pt}-\hspace{-2pt}z_{li})(x_{i}\hspace{-2pt}-\hspace{-2pt}z_{li})\hspace{-2pt}+\hspace{-2pt}\frac{1}{2}(x_{i}\hspace{-2pt}-\hspace{-2pt}z_{li})^{2}.
%\vspace{-11pt}
\end{split}
\end{equation}
It can be noted that %$\frac{\lambda}{L} \sum_{j=0}^{J}\beta_{j}w_{ji}=\frac{\lambda}{L}$ due to the constraint $\sum_{j=0}^{J}w_{ji}=1$ and
$-(v_{i}\hspace{-2pt}-\hspace{-2pt}z_{li})(x_{i}\hspace{-2pt}-\hspace{-2pt}z_{li})\geq -|v_{i}\hspace{-2pt}-\hspace{-2pt}z_{li}||x_{i}\hspace{-2pt}-\hspace{-2pt}z_{li}|$. Thus the inequality in \eqref{n-l1-proximalOperatorElementLemmaInequality} is equivalent to:
%\begin{eqnarray}
%\label{n-l1-proximalOperatorElementLemmaInequalityFurther}
%% \nonumber to remove numbering (before each equation)
%  h(v_{i}) &\geq& \frac{\lambda}{L} |v_{i}-z_{li}|-|v_{i}-z_{li}||x_{i}-z_{li}|+ \frac{1}{2}(v_{i}-z_{li})^{2}\\
%   && -\frac{\lambda}{L} \sum\limits_{j=0}^{J}w_{ji}|z_{ji}-z_{li}| \nonumber +\frac{1}{2}(x_{i}-z_{li})^{2}.
%\end{eqnarray}
%\vspace{-8pt}
\begin{equation}
\label{n-l1-proximalOperatorElementLemmaInequalityFurther}
\begin{split}
  \hspace{-2pt}h(v_{i})\hspace{-2pt} \geq &|v_{i}\hspace{-2pt}-\hspace{-2pt}z_{li}|\frac{\lambda}{L}\hspace{-4pt}\sum\limits_{j=0}^{J}\hspace{-2pt}\beta_{j}w_{ji} \hspace{-2pt}-\hspace{-2pt}|v_{i}\hspace{-2pt}-\hspace{-2pt}z_{li}||x_{i}\hspace{-2pt}-\hspace{-2pt}z_{li}|+ \frac{1}{2}(v_{i}\hspace{-2pt}-\hspace{-2pt}z_{li})^{2}\\[-10pt]
  %\vspace{-9pt}
    &-\frac{\lambda}{L} \sum\limits_{j=0}^{J}\beta_{j}w_{ji}|z_{ji}\hspace{-2pt}-\hspace{-2pt}z_{li}| \hspace{-2pt}+\hspace{-2pt}\frac{1}{2}(x_{i}\hspace{-2pt}-\hspace{-2pt}z_{li})^{2}.
    %\vspace{-11pt}
\end{split}
\end{equation}
Without difficulty, from the expression in \eqref{n-l1-proximalZ}, we get:
%$\sum\limits_{j=0}^{J} \hspace{-2pt}\beta_{j}w_{ji}( \hspace{-2pt}-1 \hspace{-2pt})^{b(l-1<j)}\geq-\sum\limits_{j=0}^{J}\beta_{j}w_{ji}$ and $\sum\limits_{j=0}^{J} \hspace{-2pt}\beta_{j}w_{ji}( \hspace{-2pt}-1 \hspace{-2pt})^{b(l<j)} \hspace{-2pt} \leq \hspace{-2pt} \sum\limits_{j=0}^{J} \hspace{-2pt}\beta_{j}w_{ji}$.
%from the condition in \eqref{n-l1-proximalZ},
% we rewrite as:
%\vspace{-12pt}
% \begin{equation}\label{n-l1-proximalZrewrite}
%\frac{\lambda}{L} \sum\limits_{j=0}^{J}w_{ji}(-1)^{b(l-1<j)}\leq x_{i}-z_{li}\leq \frac{\lambda}{L} \sum\limits_{j=0}^{J}w_{ji}(-1)^{b(l<j)}.
%    \vspace{-15pt}
%\end{equation}
%$Using these inequalities with the expression \eqref{n-l1-proximalZ}, we get:
%\vspace{-8pt}
 \begin{equation}\label{n-l1-proximalZrewriteMore}
 \begin{split}
\hspace{-5pt}-\hspace{-1pt}\frac{\lambda}{L}\hspace{-4pt}\sum\limits_{j=0}^{J}\hspace{-2pt}\beta_{j}w_{ji}\hspace{-2pt}\leq \hspace{-2pt} x_{i}\hspace{-2pt}-\hspace{-2pt}z_{li}\hspace{-2pt}\leq \hspace{-2pt}\frac{\lambda}{L}\hspace{-4pt}\sum\limits_{j=0}^{J}\hspace{-2pt}\beta_{j}w_{ji}
\hspace{-2pt}\Leftrightarrow \hspace{-2pt}|x_{i}\hspace{-2pt}-\hspace{-2pt}z_{li}|\hspace{-2pt}\leq \hspace{-2pt}\frac{\lambda}{L}\hspace{-4pt}\sum\limits_{j=0}^{J}\hspace{-2pt}\beta_{j}w_{ji}.
    %\vspace{-10pt}
\end{split}
\end{equation}
Eventually, we observe that the part including $v_{i}$ in the right hand side of the inequality $h(v_{i})$ in \eqref{n-l1-proximalOperatorElementLemmaInequalityFurther} is
   %\vspace{-6pt}
 \begin{equation}\label{leftInequality}
   |v_{i}-z_{li}|\Big(\frac{\lambda}{L}\sum\limits_{j=0}^{J}\beta_{j}w_{ji}\hspace{-2pt}-\hspace{-2pt}|x_{i}-z_{li}|\Big)\hspace{-2pt}+\hspace{-2pt} \frac{1}{2}(v_{i}\hspace{-2pt}-\hspace{-2pt}z_{li})^{2}.
      %\vspace{-7pt}
 \end{equation}
 With \eqref{n-l1-proximalZrewriteMore}, the expression in \eqref{leftInequality} is minimum when $v_{i}\hspace{-2pt}=\hspace{-2pt}z_{li}$. Therefore, we deduce that $h(v_{i})$ \eqref{n-l1-proximalOperatorElement} is minimum when $v_{i}\hspace{-2pt}=\hspace{-2pt}z_{li}$.
\end{proof}
%\vspace{-7pt}
In summary, from \eqref{partialZero} with conditions in \eqref{n-l1-proximalX}, \eqref{n-l1-proximalZ} and Lemma \ref{lemmaMin}, we obtain:
   %\vspace{-5pt}
\begin{equation}\label{n-l1-proximalOperatorElementCompute}
\Gamma_{\hspace{-2pt}\frac{1}{L}g}(x_{i}) \hspace{-2pt}= \hspace{-4pt}\left\{\hspace{-1pt}
\begin{array}{l}
\hspace{-5pt}x_{i}\hspace{-2pt}-\hspace{-2pt}\frac{\lambda}{L} \hspace{-2pt} \sum\limits_{j=0}^{J}\hspace{-2pt}\beta_{j}w_{ji}(\hspace{-2pt}-1\hspace{-2pt})^{\mathfrak{b}(l<j)}\mathrm{~if\hspace{1pt}}\eqref{n-l1-proximalX};\\[-3pt]
\hspace{-5pt}z_{li}~~~~~~~~~~~~~~~~~~~~~~~~~~~~~~~~~~~~~~\mathrm{if\hspace{1pt}} \eqref{n-l1-proximalZ};
\end{array}
\right.
   %\vspace{-5pt}
\end{equation}
%   \vspace{-5pt}
%\begin{subequations}
%\label{eq:Parent}
%\begin{align}
%&z_{li}\hspace{-2pt}+\hspace{-2pt}\frac{\lambda}{L}\hspace{-4pt} \sum\limits_{j=0}^{J}\hspace{-4pt}\beta_{j}w_{ji}(\hspace{-2pt}-1\hspace{-2pt})^{b(l<j)}\hspace{-4pt}<\hspace{-3pt}x_{i}\hspace{-3pt}<\hspace{-4pt}z_{l+1i}\hspace{-2pt}+\hspace{-2pt}\frac{\lambda}{L}\hspace{-4pt} \sum\limits_{j=0}^{J}\hspace{-4pt}\beta_{j}w_{ji}(\hspace{-2pt}-1\hspace{-2pt})^{b(l<j)}\label{n-l1-proximalOperatorElementCompute:1};\\
%&z_{li}\hspace{-2pt}+\hspace{-2pt}\frac{\lambda}{L} \hspace{-4pt}\sum\limits_{j=0}^{J}\hspace{-4pt\beta_{j}}w_{ji}(\hspace{-2pt}-1\hspace{-2pt})^{b(l-1<j)}\hspace{-4pt}\leq \hspace{-3pt}x_{i}\hspace{-3pt}\leq \hspace{-4pt} z_{li}\hspace{-2pt}+\hspace{-2pt}\frac{\lambda}{L} \hspace{-4pt}\sum\limits_{j=0}^{J}\hspace{-4pt}\beta_{j}w_{ji}(\hspace{-2pt}-1\hspace{-2pt})^{b(l<j)}. \label{n-l1-proximalOperatorElementCompute:2}
%\end{align}
%   \vspace{-5pt}
%\end{subequations}

%\clearpage
\bibliographystyle{IEEEtran}
\bibliography{./IEEEfull,./IEEEabrv,./bibliography}

\end{document}